\documentclass{article}
\usepackage{ijcai25}

\usepackage{times}
\usepackage{soul}
\usepackage{url}
\usepackage[hidelinks]{hyperref}
\usepackage[utf8]{inputenc}
\usepackage[small]{caption}
\usepackage{graphicx}
\usepackage{amsmath}
\usepackage{amsthm}
\usepackage{amsfonts}
\usepackage{mathtools}
\usepackage{mathrsfs}\usepackage{booktabs}
\usepackage{bm}
\usepackage{booktabs}
\usepackage{algorithm}
\usepackage{algorithmic}
\usepackage{framed}
\usepackage[switch]{lineno}
\mathtoolsset{showonlyrefs}
\usepackage{float}
\usepackage{lipsum}
\usepackage{multirow}
\usepackage{appendix}

\newcommand{\bmx}{\bm{x}}
\newcommand{\bmy}{\bm{y}}
\newcommand{\bmz}{\bm{z}}

\newcommand{\bmr}{\bm{r}}

\newcommand{\bmf}{\bm{f}}
\newcommand{\ojzj}{OneJumpZeroJump}

\newtheorem{theorem}{Theorem}
\newtheorem{lemma}[theorem]{Lemma}
\newtheorem{definition}{Definition}

\title{A Theoretical Perspective on Why
 Stochastic Population Update Needs an Archive \\ in Evolutionary Multi-objective Optimization}

\author{
Shengjie Ren$^{1,2}$
\and
Zimin Liang$^{3}$\and
Miqing Li$^{3}$\And
Chao Qian$^{1,2}$
\affiliations
$^1$National Key Laboratory for Novel Software Technology, Nanjing University, Nanjing 210023, China\\
$^2$School of Artificial Intelligence, Nanjing University, Nanjing 210023, China\\
$^3$School of Computer Science, University of Birmingham,
Birmingham B15 2TT, U.K.
\emails
rensj@lamda.nju.edu.cn, zxl525@student.bham.ac.uk, m.li.8@bham.ac.uk, qianc@nju.edu.cn
}

\begin{document}

\maketitle

\begin{abstract}

  Evolutionary algorithms (EAs) have been widely applied to multi-objective optimization due to their population-based nature. Population update, a key component in multi-objective EAs (MOEAs), is usually performed in a greedy, deterministic manner. However, recent studies have questioned this practice and shown that stochastic population update (SPU), which allows inferior solutions have a chance to be preserved, can help MOEAs jump out of local optima more easily. 
  Nevertheless, SPU risks losing high-quality solutions, potentially requiring a large population. Intuitively, a possible solution to this issue is to introduce an archive that stores the best solutions ever found.
  In this paper, we theoretically show that using an archive allows a small population and may enhance the search performance of SPU-based MOEAs. We examine two classic algorithms, SMS-EMOA and NSGA-II, on the bi-objective problem OneJumpZeroJump, and prove that using an archive can reduce the expected running time upper bound (even exponentially). The comparison between SMS-EMOA and NSGA-II also suggests that the $(\mu+\mu)$ update mode may be more suitable for SPU than the $(\mu+1)$ update mode. We also validate our findings empirically.
We hope this work may provide theoretical support to explore different ideas of designing algorithms in evolutionary multi-objective optimization.

\end{abstract}

\section{Introduction}

Multi-objective optimization deals with scenarios where multiple objectives must be optimized simultaneously. They are very common in real-world applications. Since the objectives of a multi-objective optimization problem (MOP) are usually conflicting, there does not exist a single optimal solution, but instead a set of solutions which represent different optimal trade-offs between these objectives, called Pareto optimal solutions. The objective vectors of these solutions form the Pareto front. The goal of multi-objective optimization is to find the Pareto front or a good approximation of it. 

Evolutionary algorithms (EAs), a kind of randomized heuristic optimization algorithms inspired by natural evolution, have been found well-suited to MOPs due to their population-based nature. Their widespread applications are across various real-world domains~\cite{deb2001book,qian19el}. Notably, there have been developed a multitude of well-established multi-objective EAs (MOEAs), including the non-dominated sorting genetic algorithm II (NSGA-II)~\cite{deb-tec02-nsgaii}, multi-objective evolutionary algorithm based on decomposition (MOEA/D)~\cite{zhang2007moea}, and $\mathcal{S}$ metric selection evolutionary multi-objective optimization algorithm (SMS-EMOA)~\cite{beume2007sms}. 
\begin{table*}[t]
  \centering
  \begin{tabular}{lll}
  \toprule
    & \textbf{Stochastic Population Update} & \textbf{Stochastic Population Update + Archive} \\ \midrule
    \multirow{5}{*}[-1ex]{SMS-EMOA } & $O(\mu n^k \cdot \min\{1,\sqrt{k}\mu/2^{k/2}\})$~\cite{bian23stochastic} & \\
     &$O(\mu n^k \cdot \min\{1,\mu/2^{k}\})$~\cite{zheng2024sms} & \\
    & [$\mu\ge 2(n-2k+4)$; $p_s=1/2$] & $O( \mu n^k\cdot \min\{1,(e\ln C/k)^{k-1}\}) $ [Theorem~\ref{thm:sms2}]\\
    \cmidrule(lr){2-2} 
       & $O( \mu n^k\cdot \min\{1,(e\ln C/k)^{k-1}\})$ [Theorem~\ref{thm:sms1}] & [$\mu\ge 3$; $C =e\mu/(p_s(1-p_c))$] \\ 
       & [$\mu\ge (n-2k+4)/(1-p_s)$; $C = e\mu/(p_s(1-p_c))$] & \\ \midrule
    \multirow{4}{*}[-0.5ex]{NSGA-II} & $O(\mu \sqrt{k}(n/2)^k)$~\cite{bian23stochastic}& \\
    & [$\mu\ge 8(n-2k+3)$; $p_s=1/4$; $k > 8e^2$] & $O( \mu n^k\cdot \min\{1,(e\ln C/k)^{k-1}\})$ [Theorem~\ref{thm:NSGA2-2}]\\
    \cmidrule(lr){2-2} 
       & $O( \mu n^k\cdot \min\{1,(e\ln C/k)^{k-1}\})$~[Theorem~\ref{thm:NSGA2-1}] & [$\mu \ge 5$; $C = e/(p_s(1-p_c))$] \\ 
       & [$\mu \ge 4(n-2k+3)/(1-2p_s)$; $C = e/(p_s(1-p_c))$] & \\ 
    \bottomrule
  \end{tabular}
  \caption{
  The expected number of fitness evaluations of SMS-EMOA and NSGA-II for solving \ojzj\ when using SPU alone, or with an archive, where $n$ denotes the problem size, $k$ ($2\le k<n/2$) denotes the parameter of \ojzj, $\mu$ denotes the population size, $p_c$ denotes the probability of using crossover, and $p_s$ denotes the proportion of the current population and offspring solution(s) that SPU selects to preserve directly. 
  The required ranges of $p_c$ are: $1-p_c = \Omega(1)$ for Theorems~\ref{thm:sms1} and~\ref{thm:NSGA2-1}; $p_c = \Theta(1)$ for Theorems~\ref{thm:sms2} and~\ref{thm:NSGA2-2}. The required ranges of $p_s$ are: $p_s \in [1/(\mu+1), 1-o(1/\mu))$ for Theorem~\ref{thm:sms1}; $p_s \in [1/(\mu+1), (\mu-2)/(\mu+1)]$ for Theorem~\ref{thm:sms2}; $p_s \in [1/(2\mu), 1/2-o(1/\mu))$ for Theorem~\ref{thm:NSGA2-1}; $p_s \in [1/(2\mu), (\mu-4)/(2\mu)]$ for Theorem~\ref{thm:NSGA2-2}.
  }
  \label{tab:summary}
\vspace{-0.1cm}
\end{table*}

In MOEAs, a key component is population update (aka environmental selection or population maintenance). It aims to select a set of promising solutions from the current population and newly generated solutions, which serves as a reservoir to generate high-quality solutions in subsequent generations. In most existing MOEAs, the population update is performed in a greedy and deterministic manner, with the best solutions (non-dominated solutions) always being preserved. This is based on the assumption that higher-quality solutions are more likely to generate better offspring. However, this is not always the case, particularly in rugged problem landscapes with many local optima, where solutions can easily get trapped in MOEAs. Repetitively exploring such local-optimal solutions may not help. Indeed, recent studies show that mainstream MOEAs (e.g., NSGA-II and SMS-EMOA) can easily stagnate, and even more, their population may end up in a very different area at a time~\cite{li2023moeas}.

Very recently, Bian \textit{et al.}~\shortcite{bian23stochastic} analytically showed that introducing randomness in the population update process of MOEAs (called stochastic population update, SPU) can help the search.
Specifically, the study proved that for SMS-EMOA solving the common benchmark problem OneJumpZeroJump, when $k=n/2-\Omega(n)$, using SPU can bring an acceleration of $\Omega(2^{k/2}/(\sqrt{k}\mu^2))=\Omega(2^{k/2}/(\sqrt{k}(n-2k+4)^2))$ 
on the expected running time, where $n$ denotes the problem size, $k$ ($2\le k<n/2$) denotes the parameter of \ojzj, and $\mu$ denotes the population size. Subsequently, Zheng and Doerr~\shortcite{zheng2024sms} extended SMS-EMOA to solve a many-objective problem, $m$\ojzj, showing that the same SPU can bring an acceleration of $\Theta(2^{k}/\mu)=\Theta(2^{k}/(2n/m-2k+3)^{m/2})$ as well, where $m$ is the number of objectives. These works echoed the empirical studies that show the benefit of considering non-elitism in MOEAs~\cite{tanabe2019non,li2023nonelitist}.

SPU in~\cite{bian23stochastic,zheng2024sms} introduces randomness by randomly selecting a proportion $p_s$ of the combined set of the current population and offspring solution(s) to be directly preserved into the next generation. This essentially gives a chance for inferior solutions to survive, which enables the evolutionary search to go along inferior regions which may be close to Pareto optimal regions. However, a cost of this method is that there is less space for the best solutions in the population. Some very best solutions to the problem (e.g., globally non-dominated solutions), even found by an MOEA, may be discarded during the population update process. This necessitates a large population used.
Unfortunately, when the population size is large, the benefit of using SPU may vanish. This is because the benefit of SPU comes from the operation on inferior solutions while the large population size will lead to a small probability of selecting these solutions. For example, for SMS-EMOA solving \ojzj, the acceleration of $\Omega(2^{k/2}/(\sqrt{k}\mu^2))$ brought by SPU~\cite{bian23stochastic} will vanish when the population size $\mu$ is exponential w.r.t. $k$, e.g., $k = \log n$; for SMS-EMOA solving $m$OneJumpZeroJump, the acceleration of $ \Theta(2^k/(2n/m - 2k+3)^{m/2})$~\cite{zheng2024sms} will vanish when the number $m$ of objectives is large, e.g., $m\ge k$, because the population size $\mu=(2n/m - 2k+3)^{m/2}$ increases rapidly with $m$.
This dilemma impacts both the effectiveness and practicality of using SPU in MOEAs.

Intuitively, a possible solution to this issue is to use an archive to store the best solutions ever found. In fact, in the area of MOEAs, this approach has become a popular practice~\cite{li2023multi}. Since the formalization of the archiving problem in the early 2000s~\cite{knowles2003properties}, there has been increasing interest and feasibility to use (even unbounded) archives in MOEAs, as seen in e.g.~\cite{fieldsend2003using,krause2016unbounded,brockhoff2019benchmarking,ishibuchi2020new}.
In this paper, we analytically show that incorporating an unbounded archive into SPU can reduce the population size and significantly enhance acceleration.
Specifically, we compare the expected running time of two well-established MOEAs, SMS-EMOA and NSGA-II, with SPU for solving OneJumpZeroJump, when an archive is used or not. The results are shown in Table~\ref{tab:summary}. Our contributions can be summarized as follows.
\begin{itemize}
  \item We theoretically show that incorporating an archive mechanism with SPU can reduce the upper bound on the expected running time (even exponentially). 
For example, comparing Theorems~\ref{thm:sms1} and~\ref{thm:sms2} in Table~\ref{tab:summary}, the expected running time of SMS-EMOA with SPU for solving \ojzj, no matter whether an archive is used or not, is $O(\mu n^k \cdot \min\{1,(e\ln C/k)^{k-1}\})$, where $C = e\mu/(p_s(1-p_c))$. The key difference is that using an archive allows a constant population size, resulting in a significantly smaller $C$ and thus reducing the upper bound significantly.
  Note that Bian \textit{et al.}~\shortcite{bian2024archive} recently proved the effectiveness of using an archive (bringing polynomial acceleration) for MOEAs, while our analysis for MOEAs with SPU reveals that even exponential acceleration can be obtained.
  \item Comparing the results of Theorems~\ref{thm:sms1} and~\ref{thm:NSGA2-1} in Table~\ref{tab:summary}, we can find that the upper bound of NSGA-II is smaller than SMS-EMOA when using SPU.  
  Our analysis reveals that the benefit of NSGA-II is due to its $(\mu+\mu)$ update mode, which selects each solution in the current population for reproduction and thus makes exploring promising dominated solutions easier. 
  \item In addition, our derived running time bounds for MOEAs with SPU in the second column of Table~\ref{tab:summary} are significantly better than the previously known ones~\cite{bian23stochastic,zheng2024sms}. This improvement stems from the analysis method of considering both the number and size of jumps across gap of dominated solutions. 
  The method happens to share similarities with that of proving Lemma~12 in~\cite{doerr2024hyper}.
  Moreover, our running time bounds are more general, as we consider variable survival probability $p_s$ and crossover probability $p_c$.
\end{itemize}

We also validate our theoretical findings through an empirical study on the artificial OneJumpZeroJump problem and the multi-objective travelling salesperson problem (MOTSP)~\cite{ribeiro2002study}. The results show that combining SPU with an archive leads to the best performance for both SMS-EMOA and NSGA-II. Furthermore, the results on the MOTSP show that with the SPU method, NSGA-II always performs better than SMS-EMOA. These results confirm our theoretical findings.

Finally, we give a brief overview about the running time analysis of MOEAs. Over the last decade, there has been an increasing interest for the evolutionary theory community to study MOEAs. Early theoretical works~\cite{laumanns-nc04-knapsack,LaumannsTEC04,Neumann07,Giel10,Neumann10,doerr2013lower,Qian13,qian-ppsn16-hyper,bian2018tools} mainly focus on analyzing the expected running time of a simple MOEA like GSEMO/SEMO. Recently, researchers have begun to examine practical MOEAs. Huang \textit{et al.}~\shortcite{huang2021runtime} investigated MOEA/D, assessing the effectiveness of different decomposition methods. Zheng \textit{et al.}~\shortcite{zheng2021first} conducted the first theoretical analysis of NSGA-II. Bian \textit{et al.}~\shortcite{bian23stochastic} analyzed the running time of SMS-EMOA and showed that SPU can bring acceleration.
Moreover, Wietheger and Doerr~\shortcite{wietheger23nsgaiii} demonstrated that NSGA-III~\cite{deb2014nsgaiii} exhibits superior performance over NSGA-II in solving the tri-objective problem $3$OneMinMax. Ren \textit{et al.}~\shortcite{Ren2024spea2} analyzed the running time of SPEA2 on three $m$-objective problems.
Some other works on well-established MOEAs include~\cite{bian2022better,zheng2022current,zheng2023manyobj,cerf2023first,dang2023analysing,dang2023crossover,doerr2023lower,doerr2023crossover,doerr24simple,doerr24unbounded,lu2024imoea,Opris2024nsgaiii,opris2024crossover,ren2024multimodel}.

\section{Multi-objective Optimization}



Multi-objective optimization aims to optimize two or more objective functions simultaneously, as shown in Definition~\ref{def_MO}. In this paper, we focus on maximization, while minimization can be defined similarly. 
The objectives are typically conflicting, meaning that there is no canonical complete order in the solution space $\mathcal{X}$. To compare solutions, we use the \emph{domination} relationship in Definition~\ref{def_Domination}.
A solution is \emph{Pareto optimal} if no other solution in $\mathcal{X}$ dominates it. The set of objective vectors corresponding to all Pareto optimal solutions constitutes the \emph{Pareto front}. The goal of multi-objective optimization is to find the Pareto front or its good approximation.
\begin{definition}[Multi-objective Optimization]\label{def_MO}
	Given a feasible solution space $\mathcal{X}$ and objective functions $f_1,f_2,\ldots, f_m$, multi-objective optimization can be formulated as
	\[
	\max_{\bmx\in
		\mathcal{X}}\bmf(\bmx)=\max_{\bmx \in
		\mathcal{X}} \big(f_1(\bmx),f_2(\bmx),...,f_m(\bmx)\big).
	\]
\end{definition}
\begin{definition}\label{def_Domination}
	Let $\bm f = (f_1,f_2,\ldots, f_m):\mathcal{X} \rightarrow \mathbb{R}^m$ be the objective vector. For two solutions $\bmx$ and $\bmy\in \mathcal{X}$:
	\begin{itemize}
		\item $\bmx$ \emph{weakly dominates} $\bmy$ (denoted as $\bmx \succeq \bmy$) if for any $1 \leq i \leq m, f_i(\bmx) \geq f_i(\bmy)$;
		\item $\bmx$ \emph{dominates} $\bmy$ (denoted as $\bmx\succ \bmy$) if $\bm{x} \succeq \bmy$ and $f_i(\bmx) > f_i(\bmy)$ for some $i$;
		\item $\bmx$ and $\bmy$ are \emph{incomparable} if neither $\bmx\succeq \bmy$ nor $\bmy\succeq \bmx$.
	\end{itemize}
\end{definition}
\noindent Note that the notions of ``weakly dominate" and ``dominate" are also called ``dominate" and ``strictly dominate" in some works~\cite{cerf2023first,wietheger23nsgaiii}.



Next, we introduce the benchmark problem \ojzj\ studied in this paper. The \ojzj\ problem as presented in Definition~\ref{def:ojzj} is constructed based on the Jump problem~\cite{doerr2020theory}, and has been widely used in MOEAs' theoretical analyses~\cite{doerr2021ojzj,doerr2023nsgaojzj,lu2024imoea,ren2024multimodel}. Its first objective is the same as the Jump problem, which aims to maximize the number of 1-bits of a solution, except for a valley around $1^n$ (the solution with all 1-bits) where the number of 1-bits should be minimized. The second objective is isomorphic to the first, with the roles of 1-bits and 0-bits reversed. The Pareto front of the \ojzj \ problem is
$ \{(a, n+2k-a)\mid a\in[2k..n]\cup\{k, n+k\}\}$, whose size is $n-2k+3$, and the Pareto optimal solution corresponding to $(a, n+2k-a)$, $a\in[2k..n]\cup\{k, n+k\}$, is any solution with $(a-k)$ 1-bits. We use $F_{I}^*=\{(a, n+2k-a)\mid a\in[2k..n]\}$ to denote the inner part of the Pareto front. 

\begin{definition}[\ojzj~\cite{doerr2021ojzj}]\label{def:ojzj}
	The \ojzj \ problem is to find $n$ bits binary strings which maximize
	\[ f_1(\bmx) = \begin{cases}
		k+|\bmx|_1, & \text{if }|\bmx|_1 \leq n-k\text{ or } \bmx=1^n,\\
		n-|\bmx|_1, & \text{else},
	\end{cases}\]
	\[f_2(\bmx) = \begin{cases}
		k+|\bmx|_0, & \text{if }|\bmx|_0 \leq n-k\text{ or } \bmx=0^n,\\
		n-|\bmx|_0, & \text{else},
	\end{cases}\]
	where $k\in \mathbb{Z} \wedge 2\le k<n/2$, and $|\bmx|_1$ and $|\bmx|_0$ denote the number of 1-bits and 0-bits in $\bmx \in \{0,1\}^n$, respectively.
\end{definition}

\section{Running Time Analysis of SMS-EMOA}

In this section, we prove that using an archive can bring (even exponential) speedup for the well-established MOEA, SMS-EMOA~\cite{beume2007sms}, with stochastic population update (SPU) solving the \ojzj\ problem. In Section~\ref{subsec-sms-intro}, we first introduce SMS-EMOA with SPU, and its expected running time for solving \ojzj, which is much tighter than the previous results~\cite{bian23stochastic,zheng2024sms}. Then, in Section~\ref{subsec-sms-analysis}, we prove that using an archive can significantly reduce the upper bound on the expected running time. Due to space limitation, all mathematical proofs could only be sketched or had to be omitted completely.

\subsection{SMS-EMOA with SPU}\label{subsec-sms-intro}

SMS-EMOA presented in Algorithm~\ref{alg:sms} is a popular MOEA, which employs non-dominated sorting and hypervolume indicator to update the population. It starts from an initial population of $\mu$ solutions (line~1). In each generation, it randomly selects a solution $\bm{x}$ from the current population (line~3) for reproduction. With probability $p_c$, it selects another solution $\bmy$ and applies one-point crossover on $\bmx$ and $\bmy$ to generate an offspring solution $\bmx'$ (lines~4--7); otherwise, $\bmx'$ is set as the copy of $\bmx$ (line~9). 
Note that one-point crossover selects a random point $i \in \{1, 2, \ldots, n\}$ and exchanges the first $i$ bits of two parent solutions, which actually produces two new solutions, but the algorithm only picks the one that consists of the first part of the first parent solution and the second part of the second parent solution.
Afterwards, bit-wise mutation flips each bit of $\bmx'$ with probability $1/n$ to produce an offspring $\bmx''$ (line~11).
Then, the worst solution in $P \cup \{\bm{x}''\}$, the union of the current population and offspring, is removed (line~12) using the \textsc{Population Update of SMS-EMOA} subroutine described in Algorithm~\ref{alg:smspopdate}. The subroutine first partitions the solution set $Q$ (where $Q = P \cup \{\bm{x}''\}$) into non-dominated sets $R_1, R_2, \ldots, R_v$ , where $R_1$ contains all non-dominated solutions in $Q$, and $R_i$ ($i \ge 2$) contains all non-dominated solutions in $Q \setminus \cup_{j=1}^{i-1} R_j$. A solution $\bm{z} \in R_v$ is then removed by minimizing $\Delta_{\bmr}(\bm{x}, R_v) := HV_{\bmr}(R_v) - HV_{\bmr}(R_v \setminus \{\bm{x}\})$, where $HV_{\bmr}(X)$ denotes the hypervolume of the solution set $X$ with respect to a reference point $\bmr \in \mathbb{R}^m$ ($\forall i, r_i\le \min_{\bmx\in \mathcal{X}}f_i(\bmx)$), i.e., the volume of the objective space between the reference point and the objective vectors of the solution set. A larger hypervolume indicates better approximation of the Pareto front in terms of convergence and diversity. For bi-objective problems, as defined in the original SMS-EMOA~\cite{beume2007sms}, the algorithm omits the reference point and directly preserves the two boundary points, allowing the hypervolume to be calculated accordingly.

In~\cite{bian23stochastic}, the SPU method is introduced and shown to be beneficial for the search of MOEAs. 
During population updates, SPU randomly selects a proportion $p_s$ of the current population and the offspring solution(s) to directly survive into the next generation and the removed part is selected from the rest solutions. 
This implies that each solution, including the worst solution in the population, has at least a probability $p_s$ of surviving to the next generation. Specifically, \textsc{SPU of SMS-EMOA} as presented in Algorithm~\ref{alg:sms-SPU} is used to replace the original \textsc{Population Update} procedure in line~12 of Algorithm~\ref{alg:sms}. Note that $p_s$ is set to $1/2$ in~\cite{bian23stochastic}, while we consider a general $p_s$ here. 
\begin{algorithm}[t!]
	\caption{SMS-EMOA}
	\label{alg:sms}
	\textbf{Input}: objective function $f_1,f_2\cdots,f_m$, population size $\mu$, probability $p_c$ of using crossover \\
	\textbf{Output}: $\mu$ solutions from $\{0,1\}^n$
	\begin{algorithmic}[1] 
		\STATE $P\leftarrow \mu$ solutions uniformly and randomly selected from $\{0, 1\}^{\!n}$ with replacement;
		\WHILE{criterion is not met}
		\STATE select a solution $\bmx$ from $P$ uniformly at random;
		\STATE sample $u$ from the uniform distribution over $[0, 1]$;
		\IF{$u<p_c$}
		\STATE select a solution $\bmy$ from $P$ uniformly at random;
		\STATE apply one-point crossover on $\bmx$ and $\bmy$ to generate  $\bmx'$
		\ELSE 
		\STATE set $\bmx'$ as the copy of $\bmx$
		\ENDIF
		\STATE apply bit-wise mutation on $\bmx'$ to generate $\bmx''$;
		\STATE $P\leftarrow \textsc{Population Update}(P\cup \{\bmx''\})$
		\ENDWHILE
		\RETURN $P$
	\end{algorithmic}
\end{algorithm}
\begin{algorithm}[t]
	\caption{\textsc{Population Update of SMS-EMOA} ($Q$)~}
	\label{alg:smspopdate}
	\textbf{Input}: a set $Q$ of solutions, and a reference point $\bmr\in \mathbb{R}^m$\\
	\textbf{Output}: $|Q|-1$ solutions from $Q$
	\begin{algorithmic}[1] 
		\STATE partition $Q$ into non-dominated sets $R_1,R_2,\ldots,R_v$;
		\STATE let $\bmz=\arg\min_{\bmx\in R_v}\Delta_{\bmr}(\bmx,R_v)$;
    \RETURN $Q\setminus \{\bmz\}$
	\end{algorithmic}
\end{algorithm}
\begin{algorithm}[t!]
	\caption{ \textsc{SPU of SMS-EMOA} ($Q$)~}
	\label{alg:sms-SPU}
	\textbf{Input}: a set $Q$ of solutions, and a reference point $\bmr\in \mathbb{R}^m$\\
	\textbf{Output}:  $|Q|-1$ solutions from $Q$
	\begin{algorithmic}[1] 
		\STATE $Q'\leftarrow \lfloor |Q|\cdot (1-p_s)\rfloor$ solutions uniformly and randomly selected from $Q$ without replacement;
		\STATE partition $Q'$ into non-dominated sets $R_1,R_2,\ldots,R_v$;
		\STATE let $\bmz=\arg\min_{\bmx\in R_v}\Delta_{\bmr}(\bmx,R_v)$;
		\RETURN $Q\setminus \{\bmz\}$
	\end{algorithmic}
\end{algorithm}

The expected running time of SMS-EMOA with SPU for
solving OneJumpZeroJump has been proven to be $O(\mu n^k \cdot \min\{1,\sqrt{k}\mu/2^{k/2}\})$~\cite{bian23stochastic}, which is better than that, i.e., $O(\mu n^k)$, of the original SMS-EMOA. Intuitively, by 
by introducing randomness into the population update, the evolutionary search has a chance to go along inferior regions which are close to Pareto optimal regions, thereby making the search easier.
Here, we re-prove a tighter upper bound on the expected running time of SMS-EMOA with SPU for solving OneJumpZeroJump, as shown in Theorem~\ref{thm:sms1}. It is also more general, as it considers a survival probability $p_s \in [1/(\mu+1), 1 - o(1/\mu))$, 
rather than just $p_s = 1/2$ 
as in~\cite{bian23stochastic}. Note that the running time of EAs is often measured by the number of fitness evaluations.

\begin{theorem}\label{thm:sms1}
  For SMS-EMOA solving OneJumpZeroJump with $n-2k=\Omega(n)$, if using SPU with survival probability $p_s\in [1/(\mu+1),1-o(1/\mu))$, the crossover probability $1-p_c = \Omega(1)$, and the population size $\mu\ge (n-2k+4)/(1-p_s)$, then the expected running time for finding the whole Pareto front is $O(\mu n^k\cdot \min\{1,(e\ln C/k)^{k-1}\})$, where $C = e\mu/(p_s(1-p_c))$.
\end{theorem} 

The proof of Theorem~\ref{thm:sms1} needs Lemma~\ref{lemma:sms1}, which shows that given a proper value of $\mu$, an objective vector on the Pareto front will always be maintained once it has been found. The reason is that in Algorithm~\ref{alg:sms-SPU}, the removed solution is selected from $\left\lfloor (\mu+1)\cdot(1-p_s) \right\rfloor \geq n-2k+4$ solutions in $Q$. For each objective vector on the Pareto front, whose size is $n-2k+3$, only one solution has positive $\Delta$-value. This ensures that these solutions rank among the top $n-2k+3$ solutions, and thus will not be removed. 

\begin{lemma}\label{lemma:sms1}
For SMS-EMOA solving OneJumpZeroJump, if using SPU with survival probability $p_s \in [ 1/(\mu+1),1-o(1/\mu) )$, and the population size $\mu \ge (n-2k+4)/(1-p_s)$, then an objective vector $\bm{f}^*$ on the Pareto front will always be maintained once it has been found. 
\end{lemma}

\begin{proof}
Suppose the objective vector $(a,n+2k-a), a\in[2k\dots n] \cup \{k,n+k\}$, on the Pareto front is obtained by SMS-EMOA, i.e., there exists at least one solution in $Q$ (i.e., $P\cup \{\bm{x}'\}$ in line~12 of Algorithm~\ref{alg:sms}) corresponding to the objective vector $(a,n+2k-a)$. 
Because only one solution is removed in each generation by Algorithm~\ref{alg:sms}, we only need to consider the case that exactly one solution (denoted as $\bm{x}^*$) corresponds to the objective vector $(a,n+2k-a)$. 
Since $\bm{x}^*$ cannot be dominated by any other solution, we have $\bm{x}^* \in R_1$ in the \textsc{SPU of SMS-EMOA} procedure. Moreover, $R_1$ cannot contain solutions whose number of 1-bits are in $[1.. k-1]\cup [n-k+1 .. n-1]$, because these solutions must be dominated by $\bm{x}^*$. We also have $\Delta_{\bm{r}}(\bm{x}^*,R_1)=HV_{\bm{r}}(R_1)-HV_{\bm{r}}(R_1\setminus \{\bm{x}^*\})>0$ because the region 
\begin{equation}\label{eq:sms-ojzj-lem-1}
  \{\bm{f}'\in \mathbb{R}^2\mid a-1< f'_1\le a, n+2k-a-1<f'_2\le n+2k-a\}
\end{equation}
cannot be covered by any objective vector in $\bm{f}(\{0,1\}^n)\setminus \{(a,n+2k-a)\}$. If at least two solutions in $R_1$ have the same objective vector, they must have a zero $\Delta$-value, because removing one of them will not decrease the hypervolume covered. Thus, for each objective vector $(b,n+2k-b), b\in[2k.. n] \cup \{k,n+k\}$, at most one solution can have a $\Delta$-value larger than zero, implying that there exist at most $n-2k+3$ solutions in $R_1$ with $\Delta>0$. Note that the removed solution is chosen from $\left\lfloor (\mu+1)\cdot(1-p_s) \right\rfloor \ge n-2k+4$ solutions in $Q$. Thus, $\bm{x}^*$ will not be removed because it is one of the best $n-2k+3$ solutions.
\end{proof}

The proof idea of Theorem~\ref{thm:sms1} is to divide the optimization procedure into two phases, where the first phase aims at finding the whole inner part $F_I^*$ of the Pareto front, and the second phase aims at finding the two extreme Pareto optimal solutions $1^n$ and $0^n$. By using SPU, the population can preserve some dominated solutions to gradually reach $1^n$ (and $0^n$) instead of generating them directly from the solutions in $F_I^*$ through mutation. 

\begin{proof}[Proof of Theorem~\ref{thm:sms1}] 
We divide the optimization procedure into two phases:
the first phase starts after initialization and finishes until the inner part $F_I^*$ of the Pareto front is found; the second phase starts after the first phase and finishes until the extreme Pareto optimal solution $1^n$ is found. Note that the analysis for finding $0^n$ holds similarly.

Since the objective vectors in $ F_I^* $ will always be maintained by Lemma~\ref{lemma:sms1}, we could use the conclusion of Theorem~1 in~\cite{bian23stochastic}, which proves that the expected number of fitness evaluations of the first phase is $ O(\mu (n \log n + k^k)) $, where the term $O(\mu k^k)$ is the expected number of fitness evaluations for finding one objective vector in $F_I^*$ when event $E$, i.e., any solution in the initial population has at most $(k-1)$ or at least $(n-k+1)$ 1-bits, happens.
By Chernoff bound and $n-2k=\Omega(n)$, an initial solution has at most $(k-1)$ or at least $(n-k+1)$ 1-bits with probability $\exp(-\Omega(n))$, implying that event $E$ happens with probability $\exp(-\Omega(n))^{\mu}$. Thus, the term $O(\mu k^k)$ can actually be omitted, implying that the expected number of fitness evaluations of the first phase is $O(\mu n \log n)$. 

Next, we consider the second phase. By employing SPU, any solution (including dominated ones) can survive into the next generation with a probability of at least $p_s$. This means that the population can preserve some dominated solutions to gradually reach $1^n$. We assume a ``jump" to be an event where a solution $\bm{x}$ with $|\bm{x}|_1 \in [n-k .. n-1]$ is selected, and a new dominated solution closer to $1^n$ is generated and preserved. Thus, $1^n$ can be reached more easily through multiple jumps across the gap of dominated solution set (i.e., $\{\bm{x} \mid |\bm{x}|_1 \in [n-k+1 .. n-1]\}$). We refer to the dominated solutions along the multiple jumps as ``stepping stones" and assume that $1^n$ can be reached through $M$ stepping stones. 
Let $p_i$ denote the probability that, after selecting the $i$-th stepping-stone solution as a parent, the next stepping-stone solution is successfully generated. Next, we consider consecutive $ M+1 $ jumps for generating $1^n$. Any failure during the intermediate jumps will result in restarting the process from the beginning. The first jump starts by selecting the solution $\bm{x}$ with $n-k$ 1-bits, generating the first stepping-stone solution, and preserving it, with a probability of at least $p_0p_s/\mu$. For the $i$-th jump ($1 \leq i \leq M-1$), the $(i+1)$-th stepping-stone solution is generated from the previous one and preserved with a probability of at least $p_ip_s/\mu$. For the final jump, the Pareto optimal solution $1^n$ found must be preserved, and the jump occurs with a probability of $p_M/\mu$. Therefore, the consecutive $M+1$ jumps can happen with probability at least 
\begin{equation}\label{eq:one-breath}
\begin{aligned}
  &\frac{1}{\mu}\left(\frac{p_s}{\mu}\right)^M \prod_{i=0}^M p_i.
\end{aligned}
\end{equation}
Then, we consider the probability $p_i$ of each successful jump. Let $k_i$ denote the distance of the jump from the $i$-th stepping-stone solution to the next one. That is, the $(i+1)$-th stepping-stone solution can be generated by flipping $k_i$ 0-bits of the $i$-th stepping-stone solution from the remaining $k-\sum_{j=0}^{i-1}k_j$ 0-bits while keeping the other bits unchanged. Thus, after selecting the $i$-th stepping-stone solution, the next stepping-stone solution can be generated if crossover is not performed (whose probability is $1-p_c$), and only $k_i$ $0$-bits are flipped by bit-wise mutation (whose probability is $(1-1/n)^{n-k_i}\cdot \binom{k-\sum_{j=0}^{i-1}k_j}{k_i}/n^{k_i}$). We have
\begin{equation}\label{eq:p-i}
\begin{aligned}
  p_i &\geq (1-p_c) \cdot \bigg( 1-\frac{1}{n} \bigg)^{n-k_i} \cdot \frac{\binom{k-\sum_{j=0}^{i-1}k_j}{k_i}}{n^{k_i}}\\
  &\geq \frac{\binom{k-\sum_{j=0}^{i-1}k_j}{k_i}}{e n^{k_i}}\cdot (1-p_c),
\end{aligned}
\end{equation}
where $0\le i\le M$. Combining all $p_i$, their product
\begin{equation}\label{eq:product_p_i}
\begin{aligned}
  \prod_{i=0}^M p_i \ge \frac{(1-p_c)^{M+1}}{e^{M+1}n^k} \cdot \frac{k!}{k_0! \cdot k_1!\cdot k_2! \cdots k_M!},
\end{aligned}
\end{equation}
where the inequality holds by $\sum^M_{j=0} k_j= k$ and $\binom{k}{k_0} \binom{k-k_0}{k_1}\cdots \binom{k-\sum_{j=0}^{M-1}k_j}{k_M} = k!/(k_0! k_1! \cdots k_M!)$. By substituting Eq.~\eqref{eq:product_p_i} into Eq.~\eqref{eq:one-breath} and considering all possible jump distances, we can derive that the probability of a successful trail of $M+1$ consecutive jumps is at least
\begin{equation}
\begin{aligned}\label{eq:prob_attempts}
  &\sum_{k_0+\cdots + k_M = k}\frac{(1-p_c)}{e\mu n^k}\left(\frac{1}{C}\right)^M \cdot \frac{k!} {k_0!\cdot k_1!\cdots k_M!}\\
  &= \frac{(1-p_c)}{e\mu n^k}\left(\frac{1}{C}\right)^M \cdot\sum_{k_0+\cdots + k_M = k} \frac{k!}{k_0!\cdot k_1!\cdots k_M!}\\
  &= \frac{(1-p_c)(M+1)^k}{e\mu n^k}\left(\frac{1}{C}\right)^M,
\end{aligned}
\end{equation}
where we let $C \coloneqq e\mu/(p_s(1-p_c))$, and the last equality holds by the \textit{Multinomial Theorem}, which states that for any positive integer $k$ and non-negative integers $k_0, k_1, \ldots, k_M$ such that $k_0 + k_1 + \cdots + k_M = k$, $(a_0 + a_1 + \cdots + a_M)^k = \sum_{k_0+\cdots+k_M=k} a_0^{k_0} a_1^{k_1} \cdots a_M^{k_M}\cdot k!/(k_0!k_1!\cdots k_M!) $. When $a_0=a_1=\cdots =a_M=1$, we have $(M+1)^k = \sum_{k_0+\cdots + k_M = k} k!/(k_0! k_1!\cdots k_M!)$. Thus, the Pareto optimal solution $1^n$ can be found in at most $(e\mu n^k/(1-p_c)) \cdot C^M / (M+1)^k$ trails of $M+1$ consecutive jumps in expectation. Because each trail needs at most $M+1$ generations, 
the expected number of generations for finding $1^n$ is at most
\begin{equation}
\begin{aligned}\label{eq:runtime_with_m}
  (M+1)\cdot \frac{e\mu n^k \cdot C^M}{(1-p_c)(M+1)^k} = O\bigg(\frac{\mu n^k \cdot C^M}{(M+1)^{k-1} }\bigg),
\end{aligned}
\end{equation}
where the crossover probability $1-p_c = \Omega(1)$.

Now, we are to minimize Eq.~\eqref{eq:runtime_with_m} by setting a proper value of $M$, and thus we can get a tighter upper bound. Equivalently, our aim is to maximize
\begin{equation}\label{f_M}
  f(M) = (M+1)^{k-1}/C^M.
\end{equation}
Taking the derivative of $f(M)$ with respect to $M$, we have
$f'(M) = ((M+1)^{k-2}/C^M) (k-1-(M+1)\ln C)$. When $ k \leq e\ln C $, we take $M=0$, and then Eq.~\eqref{eq:runtime_with_m} becomes $ O(\mu n^k) $. When  
$ k > e\ln C $, and given that $C=e\mu/(p_s(1-p_c))>e$, setting $f'(M)=0$ implies that $f(M)$ reaches the maximum when $M = (k-1)/\ln C - 1$. We take $M = \lceil (k-1)/\ln C - 1 \rceil $ to ensure an integer value, and then Eq.~\eqref{eq:runtime_with_m} becomes
\begin{equation}
\begin{aligned}
&O\bigg(\frac{\mu n^k}{(\frac{k-1}{\ln C})^{k-1}} \cdot C^{\frac{k-1}{\ln C}}\bigg) =O\bigg( \mu n^k \cdot \big( \frac{e\ln C}{k} \big)^{k-1} \bigg),
\end{aligned}
\end{equation}
where the equality holds by $(k-1)/\ln C = \log_{C}^{e^{k-1}}$ and thus $C^{\frac{k-1}{\ln C}}=e^{k-1}$. Therefore, the expected number of generations of the second phase, i.e., finding $1^n$, is $O(\mu n^k\cdot \min\{1,(e\ln C/k)^{k-1}\})$, where $C = e\mu/(p_s(1-p_c))$. 

Combining the two phases, the total expected number of generations is $O(\mu n \log n + \mu n^k\cdot \min\{1,(e\ln C/k)^{k-1}\})=O(\mu n^k\cdot \min\{1,(e\ln C/k)^{k-1}\})$, where $k \geq 2$. To show the equality holds, we need to verify that $(en\ln C /k)^{k-1} \ge (e n/k)^{k-1} \ge \log n$. By defining $h(k) = (e n/k)^{k-1}$ and taking its derivative with respect to $k$, we obtain $h'(k) = (en/k)^{k-1} \left( \ln(n/k) + 1/k \right)$. Since $k < n/2$, it follows that $\ln(n/k) > 0$, making $h'(k) > 0$ and indicating that $h(k)$ increases with $k$. Given that $k \geq 2$, we have $h(k) = (e n/k)^{k-1} \geq h(2)=en/2 > \log n$. Thus, the theorem holds.	
\end{proof}


Under the same conditions as the previous results $O(\mu n^k \cdot \min\{1, \sqrt{k}\mu / 2^{k/2}\})$ in~\cite{bian23stochastic} and $O(\mu n^k \cdot \min\{1, \mu / 2^{k}\})$ in~\cite{zheng2024sms}, where SPU is used with $p_s=1/2$ and crossover probability $p_c=0$, our bound in Theorem~\ref{thm:sms1} becomes $O(\mu n^k \cdot (e \ln(2e\mu)/k)^{k-1})$. Since the bound in~\cite{zheng2024sms} is tighter than that in~\cite{bian23stochastic}, we focus on comparing our result only with~\cite{zheng2024sms}. 
When $k > 2e \ln (2e\mu)$, our bound $O(\mu n^k \cdot (e \ln (2e\mu)/k)^{k-1})$ brings an improvement ratio of $\Theta(\mu (k/(2e\ln (2e\mu)))^{k-1})$, which can be exponential when $k$ is large, e.g., $k=n/8$. 
For a very small range $\log \mu \leq k < 2e \ln(2e\mu)$, our bound shows no advantage, which is because to minimize Eq.~\eqref{eq:runtime_with_m}, we choose $M = \lceil (k-1)/\ln C-1 \rceil$ to round up to the nearest integer, leading to an over-relaxation. If we instead set $M = 1$, the same bound can be obtained. The reason of our better bound is in the second phase of finding the Pareto optimal solution $1^n$: 1)~We consider $M+1$ ``jumps" across the gap between dominated solutions, and account for all possible jump sizes to find $1^n$; 2)~We select the optimal number of jumps, i.e.,  $M=\lceil (k-1)/\ln C-1 \rceil$. Note that Bian \textit{et al.}~\shortcite{bian23stochastic} considered only two fixed-size jumps, and although Zheng and Doerr~\shortcite{zheng2024sms} accounted for all possible jump sizes, they also reduced the process to two jumps.

\subsection{An Archive is Provably Helpful}\label{subsec-sms-analysis}

In the last section, we have proved that for SMS-EMOA with SPU solving \ojzj, the expected running time is $O(\mu n^k\cdot \min\{1,(e\ln C/k)^{k-1}\})$, where $C = e\mu/(p_s(1-p_c))$. From the analysis, we find that SPU benefits evolutionary search by exploring inferior regions that are close to Pareto optimal areas, but it also requires a larger population size $\mu \geq (n-2k+4)/(1-p_s)$ to preserve the Pareto optimal solutions discovered. 
The greater the randomness introduced by SPU (i.e., the larger the value of $p_s$), the larger the population size needed. 
Note that a large population size may diminish the benefit of SPU, because it will lead to a very small probability of selecting specific dominated solutions for reproduction, which is required by SPU. For example, when the population size $\mu$ is exponential w.r.t. $k$, e.g., $k=\log n$, the improvement by SPU will vanish, compared to the expected running time $O(\mu n^k)$ without SPU~\cite{bian23stochastic}.

In this section, we theoretically show that the limitation of SPU can be alleviated by using an archive.
Once a new solution is generated, the solution will be tested if it can enter the archive. If there is no solution in the archive that dominates the new solution, then the solution will be placed in the archive, and meanwhile those solutions weakly dominated by the new solution will be deleted from the archive. Algorithmic steps incurred by adding an archive in SMS-EMOA are given as follows. In Algorithm~\ref{alg:sms}, an empty set $A$ is initialized in line~1, and
the following lines are added after line~11:
\begin{framed}\vspace{-0.8em}
  \begin{algorithmic}
  \IF{$\not \exists \bmz \in A$ such that $\bmz \succ \bmx''$}
  \STATE $A \leftarrow (A \setminus\{\bmz \in A \mid \bmx'' \succeq \bmz\}) \cup \{\bmx''\}$
  \ENDIF
\end{algorithmic}\vspace{-0.8em}
\end{framed}\noindent
The set $A$ instead of $P$ is returned in the last line. 

We prove in Theorem~\ref{thm:sms2} that if using an archive, a population size $\mu \ge 3$ is sufficient to guarantee the same running time bound of SMS-EMOA with SPU as Theorem~\ref{thm:sms1}. Note that Theorem~\ref{thm:sms1} requires $\mu \ge (n-2k+4)/(1-p_s)$, which implies that using an archive can allow a small population size and thus bring speedup. 

\begin{theorem}\label{thm:sms2}
  For SMS-EMOA solving OneJumpZeroJump with $n-2k=\Omega(n)$, if using SPU with survival probability $p_s\in [1/(\mu+1),(\mu-2)/(\mu+1)]$, the crossover probability $p_c = \Theta(1)$, the population size $\mu \ge 3$, and using an archive, then the expected running time for finding the whole Pareto front is $O(\mu n^k\cdot \min\{1,(e\ln C/k)^{k-1}\})$, where $C = e\mu/(p_s(1-p_c))$.
\end{theorem} 
\begin{proof}
   We divide the running process into three phases. The first phase starts after initialization and finishes until the two boundary solutions of $F_I^*$ (i.e., a solution with $k$ 1-bits and a solution with $n-k$ 1-bits) are found; the second phase starts after the first phase and finishes when $1^n$ and $0^n$ are both found; the third phase starts after the second phase and finishes until the whole Pareto front is found.

   For the first phase, we prove that the expected number of generations for finding a solution $\bm{x}$ with $ n-k$ 1-bits is $O(\mu n\log n)$, and the same bound holds for finding a solution with $k$ 1-bits analogously. First, we show that the maximum $f_1$ value among the Pareto optimal solutions in $P\cup \{\bm{x}''\}$ will not decrease, where $\bm{x}''$ is the offspring solution generated in each iteration. Let $J$ denote the set of Pareto optimal solutions in $P \cup \{\bm{x}''\}$ with the maximum $f_1$ value. By definition, all solutions in $J$ are Pareto optimal, which means that they all belong to $R_1$ in the non-dominated sorting procedure. Furthermore, since all solutions in $J$ have the maximum $f_1$ value, there is at least one solution $\bm{x}^* \in J$ that is a boundary solution in $R_1$, which has the highest priority for SMS-EMOA to preserve. Since the survival probability $p_s$ satisfies $p_s \in [1/(\mu+1), (\mu-2)/(\mu+1)]$, the removed solution is selected from $\lfloor (\mu+1) \cdot (1-p_s) \rfloor \geq 3$ solutions in $P \cup \{\bm{x}''\}$. Thus, if $\bm{x}^*$ is included in the survival selection, SMS-EMOA will directly preserve it, as it is one of the top two solutions. Note that for the bi-objective problem \ojzj, there can be at most two boundary solutions in $R_1$. Thus, the maximum $f_1$ value among the Pareto optimal solutions in $P\cup \{\bm{x}''\}$ will not decrease.
   
   During population initialization, as stated in Theorem~\ref{thm:sms1}, by the Chernoff bound and $n-2k=\Omega(n)$, there is at least one solution with the number of 1-bits between $k$ and $n-k$,  with a probability of $1 - \exp(-\Omega(n))^\mu$. 
   Now, we consider the increase of the maximum $f_1$ value among the Pareto optimal solutions found. In each generation, SMS-EMOA selects the Pareto optimal solution $\bm{x}$ with the maximum $f_1$ value as a parent solution (whose probability is $1/\mu$), and generates a solution with more 1-bits if crossover is not performed (whose probability is $1-p_c$) and only one of the 0-bits in $\bmx$ is flipped by bit-wise mutation (whose probability is $((n-|\bm{x}|_1)/n)\cdot (1-1/n)^{n-1}$). This implies that the probability of generating a solution with more than $|\bm{x}|_1$ 1-bits in one generation is at least
   \begin{equation}
     \frac{1}{\mu}\cdot (1-p_c)\cdot \frac{n-|\bm{x}|_1}{n}\cdot \Big(1-\frac{1}{n}\Big)^{n-1}\ge (1-p_c)\cdot \frac{n-|\bm{x}|_1}{e\mu n}.
   \end{equation}
   Thus, the expected number of generations for increasing the maximum $f_1$ value to $n$, i.e., finding a solution with $n-k$ 1-bits, is at most $(1/ (1-p_c)) \cdot \sum_{i=k}^{n-k-1} e\mu n/(n-i)=O(\mu n \log n)$, where the equality holds by $p_c=\Theta(1)$. That is, the expected number of generations of the first phase is $O(\mu n\log n)$. 

   For the second phase, the analysis is the same as that of the second phase in Theorem~\ref{thm:sms1}. Before reaching the extreme Pareto optimal solution $1^n$, the boundary Pareto optimal solution with $n-k$ 1-bits is always the Pareto optimal solution with the maximum $f_1$ value and thus preserved in the population. Using SPU, the process of reaching $1^n$ can be constructed as jumping across the gap of the dominated solution set (i.e., $\{\bm{x} \mid |\bm{x}|_1 \in [n-k+1.. n-1]\}$) through $M$ stepping-stone solutions, and any failure during the intermediate jumps will result in restarting the process from the solution with $n-k$ 1-bits. By setting proper values of $M$, this process leads to an expected number of generations, $O(\mu n^k\cdot \min\{1,(e\ln C/k)^{k-1}\})$, where $C = e\mu/(p_s(1-p_c))$.

   Finally, we consider the third phase, and will show that SMS-EMOA can find the whole Pareto front in $O(\mu n\log n)$ expected number of generations. Note that after the second phase, $0^n$ and $1^n$ must be maintained in the population $P$. Let $D=\{j \mid \exists \bm{x}\in A, |\bmx|_1 = j\}$, where $A$ denotes the archive, and we suppose $|D|=i$, i.e., $i$ points on the Pareto front have been found in the archive. Note that $i \geq 2$ as $0^n$ and $1^n$ have been found. Assume that in the reproduction procedure, a solution $\bm{x}$ with $|\bm{x}|_1 = j$ ($j \in [0..n]$) is selected as one parent. If the other selected parent solution is $1^n$, then for any $d\in [j+1..n-k] \setminus D$, there must exist a crossover point $d'$ such that exchanging the first $d'$ bits of $\bm{x}$ and $1^n$ can generate a solution with $d$ 1-bits. If the other selected parent is $0^n$, then for any $d\in [k..j-1] \setminus D$, there must exist a crossover point $d'$ such that exchanging the first $d'$ bits of $\bm{x}$ and $0^n$ can generate a solution with $d$ 1-bits. The newly generated solution can keep unchanged if no bits are flipped in bit-wise mutation. Note that the probability of selecting $1^n$ (or $0^n$) as a parent solution is $1/\mu$. Thus, the probability of generating a new point on the Pareto front is at least
\begin{equation}
\begin{aligned}
  &\frac{1}{\mu}\cdot p_c \cdot \frac{n-2k+3 - |D|}{n} \cdot \left(1-\frac{1}{n}\right)^n \\
  &\ge \frac{p_c(n-2k+3-|D|)}{e\mu n}.
\end{aligned}
\end{equation}
Then, we can derive that the expected number of generations of the third phase (i.e., for finding the whole Pareto front) is at most $\sum_{i=2}^{n-2k+2} O(\mu n/(n-2k+3-i)) = O(\mu n\log n)$.

Combining the three phases, the total expected number of generations is $O(\mu n \log n+ \mu n^k\cdot \min\{1,(e\ln C/k)^{k-1}\} + \mu n\log n) = O( \mu n^k\cdot \min\{1,(e\ln C/k)^{k-1}\})$, where $C = e\mu/(p_s(1-p_c))$. The equality holds because $k \geq 2$ and $(en\ln C /k)^{k-1} \ge \log n$, as shown in the last paragraph in the proof of Theorem~\ref{thm:sms1}. Thus, the theorem holds.	
\end{proof} 

For the case without an archive, in order to avoid losing the Pareto optimal solutions found while using SPU, the population size must be large, i.e., $\mu \geq (n-2k+4)/(1-p_s)$. However, for the case with an archive, the population size only needs to be a constant. The main reasons are: 1) Using
an archive that stores all the Pareto optimal solutions generated enables the algorithm not to worry about losing Pareto
optimal solutions, but only endeavoring to seek new Pareto optimal solutions. 2) In the context of our analysis, a constant population size is sufficient to preserve exploration-favoring solutions, i.e., the two boundary Pareto optimal solutions.

The smaller population size allows a larger probability of selecting inferior solutions that are close to Pareto optimal areas for reproduction, thus leading to speedup. 
When $k$ is limited, using only SPU may not bring acceleration. For example, when $k=e\ln (8en)$, if using SPU with $p_s = 1/2$, $\mu = 2n$, and $p_c = 1/2$, the expected running time of SMS-EMOA on OneJumpZeroJump is $O(\mu n^k \cdot \min\{1,(e\ln C/k)^{k-1}\})=O(n^{k+1})$ (where $C=e\mu/(p_s(1-p_c))=8en$), implying no acceleration compared to $O(\mu n^k)=O(n^{k+1})$ without SPU \cite{bian23stochastic}. However, by adding an archive and reducing the population size $\mu$ to $5$, while keeping other settings unchanged, the bound reduces to $O( n^k \cdot (e\ln(20e)/k)^{k-1})$, implying a superpolynomial reduction in the upper bound as $k=e\ln (8en)$. Note that we set $\mu=5$ to satisfy the condition of Theorem~\ref{thm:sms2}, i.e., to make $p_s = 1/2 \in [1/(\mu+1), (\mu-2)/(\mu+1)]$.
On the other hand, when $k$ is large, using SPU alone can lead to exponential acceleration, and the addition of an archive can further enhance this acceleration. For example, when $k = n/8$, if using SPU with $p_s = 1/2$, $\mu = 2n$, and $p_c = 1/2$, the expected running time is $O(\mu n^k \cdot (8e\ln (8en) / n)^{n/8-1})$, which implies an exponential reduction in the upper bound compared to $O(\mu n^k)$ without SPU. By adding an archive and setting $\mu$ to $5$, the bound reduces to $O(n^k \cdot (8e\ln(20e)/n)^{n/8-1})$, resulting in an improvement ratio of $\Theta(n (\ln(8en) / \ln (20e))^{n/8-1})$, i.e., an exponential reduction in the upper bound. 

We also note that the recent work~\cite{bian2024archive} has proven that using an archive can provide polynomial acceleration for MOEAs for the first time. For example, for SMS-EMOA solving OneMinMax, the expected running time is $O(\mu n \log n)$ both with and without an archive, but the archive allows for a constant $\mu$, achieving an acceleration of $\Theta(n)$. Our work gives another theoretical evidence for the effectiveness of using an archive, and further shows that superpolynomial or even exponential acceleration can be achieved. In our analysis with SPU, a sequence of $M+1$ consecutive jumps are required, which implies continuously selecting specific solutions for reproduction; using an archive allows a small population size, which increases the selection probability significantly and thus leads to greater acceleration.


\section{Running Time Analysis of NSGA-II}

We first introduce NGSA-II with SPU, and re-prove an upper bound on its expected running time for solving \ojzj, which is much tighter than the previously known bound~\cite{bian23stochastic}. Then, we prove that using an archive is helpful in this case, reducing the running time.

\subsection{NSGA-II with SPU}

The NSGA-II algorithm as presented in Algorithm~\ref{alg:nsgaii} adopts a $(\mu+\mu)$ steady state mode. It starts from an initial population of $\mu$ random solutions (line~1). In each generation, it uses fair selection to select each solution in $P$ once (the order of solution is random) to form the parent population $Q$ (line~4). Then, for each pair of the solutions in $Q$, one-point crossover and bit-wise mutation operators are applied sequentially to generate two offspring solutions $\bmx''$ and $\bmy''$ (lines~6--12), where the one-point crossover is applied with probability $p_c$.
After $\mu$ offspring solutions have been generated in $P'$, the worst $\mu$ solutions in $P\cup P'$ are removed by using the \textsc{Population Update of NSGA-II} subroutine in Algorithm~\ref{alg:nsgapopdate}. The subroutine partitions a set $Q$ into non-dominated sets $R_1,\ldots,R_v$ (line~1), where $R_1$ contains all the non-dominated solutions in $Q$ and $R_i$ ($i\ge 2$) contains all the non-dominated solutions in $Q \setminus \cup_{j=1}^{i-1} R_j$.
Then, the solutions in $R_1, R_2,\ldots, R_v$ are added into the next population, until the population size exceeds $|Q|/2$ (lines~2--4). For the critical set $R_i$ whose inclusion makes the population size larger than $|Q|/2$, the crowding distance is computed for each of the contained solutions (line~6).
Crowding distance reflects the level of crowdedness of solutions in the population. For each objective $f_j $, $1\le j\le m$, the solutions in $R_i$ are sorted according to their objective values in ascending order, and assume the sorted list is $\bmx^1,\bmx^2,\ldots,\bmx^k$. Then, the crowding distance of the solution $\bmx^l$ with respect to $f_j$ is set to $\infty$ if $l\in \{1,k\}$ and $(f_j(\bmx^{l+1})-f_j(\bmx^{l-1}))/(f_j(\bmx^k)-f_j(\bmx^1))$ otherwise.
The final crowding distance of a solution is the sum of the crowding distance with respect to each objective. 
Finally, the solutions in $R_i$ are selected to fill the remaining population slots where the solutions with larger crowding distance are preferred (line~7).

\begin{algorithm}[tb]
	\caption{NSGA-II}
	\label{alg:nsgaii}
	\textbf{Input}: objective functions $f_1,f_2\ldots,f_m$, population size $\mu$, probability $p_c$ of using crossover\\
	\textbf{Output}: $\mu$ solutions from $\{0,1\}^n$
	\begin{algorithmic}[1]
		\STATE $P\leftarrow \mu$ solutions uniformly and randomly selected from $\{0,\! 1\}^{\!n}$ with replacement;
		\WHILE{criterion is not met}
		\STATE let $P'=\emptyset, i=0$; 
      \STATE use fair selection to generate a parent population $Q$;
      \FOR{each pair of the parent solutions $\bmx$ and $\bmy$ in $Q$}
		\STATE sample $u$ from uniform distribution over $[0, 1]$;
		\IF{$u<p_c$}
		\STATE apply uniform crossover on $\bmx$ and $\bmy$ to generate two solutions $\bmx'$ and $\bmy'$
		\ELSE 
		\STATE set $\bmx'$ and $\bmy'$ as copies of $\bmx$ and $\bmy$, respectively
		\ENDIF
		\STATE apply bit-wise mutation on $\bmx'$ and $\bmy'$ to generate $\bmx''$ and $\bmy''$, respectively, and add $\bmx''$ and $\bmy''$ into $P'$
		\ENDFOR
		\STATE $P\leftarrow$ \textsc{Population Update} $(P\cup P')$;
		\ENDWHILE
		\RETURN $P$
	\end{algorithmic}
\end{algorithm}
\begin{algorithm}[t]
	\caption{\textsc{Population Update of NSGA-II} ($Q$)~}
	\label{alg:nsgapopdate}
	\textbf{Input}: a set $Q$ of solutions\\
	\textbf{Output}: $|Q|/2$ solutions from $Q$
	\begin{algorithmic}[1] 
		\STATE partition $Q$ into non-dominated sets $R_1,R_2,\ldots, R_v$;
		\STATE let $O=\emptyset$, $i=1$;
		\WHILE{$|O\cup R_i|<|Q|/2$}
		\STATE $O=O\cup R_i$, $i=i+1$
		\ENDWHILE
		\STATE assign each solution in $R_i$ with a crowding distance; 
		\STATE add $|Q|/2-|O|$ solutions in $R_i$ with the largest crowding distance into $O$;
		\RETURN $O$
	\end{algorithmic}
\end{algorithm}

Next, we introduce the stochastic population update (SPU) method for NSGA-II. During population updates, SPU randomly selects a fixed proportion $p_s$ of the current population and the offspring solutions to directly survive into the next generation and the removed part is selected from the rest solutions. This implies that each solution, including the worst solution in the population, has at least a probability $p_s$ of surviving to the next generation. Similar to the definition in~\cite{bian23stochastic}, \textsc{SPU of NSGA-II} as presented in Algorithm~\ref{alg:nsga-SPU} is used to replace the original \textsc{Population Update} procedure in line~14 of Algorithm~\ref{alg:nsgaii}. Note that $p_s$ is set to $1/4$ in~\cite{bian23stochastic}, while we consider a general $p_s$ here. 

\begin{algorithm}[t!]
	\caption{\textsc{SPU of NSGA-II} ($Q$)~}
	\label{alg:nsga-SPU}
	\textbf{Input}: a set $Q$ of solutions\\
	\textbf{Output}: $|Q|/2$ solutions from $Q$
	\begin{algorithmic}[1]
		\STATE $Q'\leftarrow \lfloor |Q|\cdot(1-p_s)\rfloor$ solutions uniformly and randomly selected from $Q$ without replacement;
		\STATE partition $Q'$ into non-dominated sets $R_1,R_2,\ldots, R_v$;
		\STATE let $O=\emptyset$, $i=1$;
		\WHILE{$|O\cup R_i|<\lfloor|Q|(1/2-p_s)\rfloor$}
		\STATE $O=O\cup R_i$, $i=i+1$
		\ENDWHILE
		\STATE assign each solution in $R_i$ with a crowding distance; 
		\STATE add $\lfloor|Q|(1/2-p_s)\rfloor-|O|$ solutions in $R_i$ with the largest crowding distance into $O$
		\RETURN $Q\setminus (Q'\setminus O)$
	\end{algorithmic}
\end{algorithm}

The expected running time of NSGA-II with SPU for solving the OneJumpZeroJump problem with $k> 8e^2$ has been proven to be $O(\mu k(n/2)^k)$~\cite{bian23stochastic}, which is better than that, i.e., $O(\mu n^k)$, of the original NSGA-II~\cite{doerr2023nsgaojzj}. Here, we re-prove a tighter upper bound on the expected running time of NSGA-II with SPU for solving OneJumpZeroJump, as shown in Theorem~\ref{thm:NSGA2-1}. It is also more general, as it considers a survival probability $p_s \in [1/2\mu, 1/2 - o(1/\mu))$, 
rather than the specific probabilities $p_s = 1/4$
analyzed in~\cite{bian23stochastic}.

\begin{theorem}\label{thm:NSGA2-1}
  For NSGA-II solving OneJumpZeroJump with $n-2k=\Omega(n)$, if using SPU with survival probability $p_s\in [1/(2\mu),1/2-o(1/\mu))$, the crossover probability $1-p_c = \Omega(1)$, and the population size $\mu\ge 4(n-2k+3)/ (1-2p_s)$, then the expected running time for finding the whole Pareto front is $O( \mu n^k\cdot \min\{1,(e\ln C/k)^{k-1}\})$, where $C = e/(p_s(1-p_c))$.
\end{theorem} 

Before proving Theorem~\ref{thm:NSGA2-1}, we first present Lemma~\ref{lemma:nsga2}, which shows that given a proper value of $\mu$, an objective vector on the Pareto front will always be maintained once it has been found

\begin{lemma}\label{lemma:nsga2}
For NSGA-II solving OneJumpZeroJump, if using SPU with survival probability of $p_s \in [ 1/2\mu,1/2-o(1/\mu) )$, and a population size $\mu$ such that $\mu \ge 4(n-2k+3)/(1-2p_s)$, then an objective vector $\bm{f}^*$ on the Pareto front will always be maintained once it has been found. 
\end{lemma}

\begin{proof}
Suppose an objective vector $\bm{f}^*$ on the Pareto front is obtained by NSGA-II, i.e., there exists at least one solution in $Q$ (i.e., $P\cup P'$ in line~14 of Algorithm~\ref{alg:nsgaii}) corresponding to the objective vector $\bm{f}^*$. Let $J$ denote the set of solutions in $P\cup P'$ with objective vector $\bmf^*$, which are selected for competition in line~1 of Algorithm~\ref{alg:nsga-SPU} (note that $J$ is a multi-set). Then, all solutions in $J$ belong to $R_1$ in the non-dominated sorting procedure, because these solutions are Pareto optimal. If the solutions in $J$ are sorted when computing the crowding distance in line~7 of Algorithm~\ref{alg:nsga-SPU}, the solution (denoted as $\bm{x}^*$) that is put in the first or the last position among solutions in $J$ will have a crowding distance greater than $0$. For solutions with the same objective vector, they are sorted according to some objective function in the crowding distance assignment procedure, and only the first or the last one has crowding distance greater than 0. Since OneJumpZeroJump has two objectives, for each objective vector, there are at most $4$ solutions having crowding distance greater than $0$. Considering that the Pareto front size of \ojzj\ is $n-2k+3$, at most $4(n-2k+3)$ solutions in $R_1$ can have crowding distance greater than $0$.

In Algorithm~\ref{alg:nsga-SPU}, $\lfloor 2\mu \cdot (1-p_s) \rfloor$ solutions are selected for competition, and $\lfloor \mu - 2\mu \cdot p_s \rfloor \ge 4(n-2k+3)$ of them will not be removed. 
Note that $\bm{x} \in J$ is among the best $4(n-2k+3)$ solutions in $R_1$, implying $\bm{x}$ must be maintained in the next population. Thus, the lemma holds.
\end{proof}

The proof idea of Theorem~\ref{thm:NSGA2-1} is similar to that of Theorem~\ref{thm:sms1}, which divides the optimization procedure into two phases, where the first phase aims at finding the whole inner part $F_I^*$ of the Pareto front, and the second phase aims at finding the two extreme Pareto optimal solutions $1^n$ and $0^n$.

\begin{proof}[Proof of Theorem~\ref{thm:NSGA2-1}]
 We divide the optimization procedure into two phases:
the first phase starts after initialization and finishes until the inner part $F_I^*$ of the Pareto front is found; the second phase starts after the first phase and finishes until $1^n$ is found. Note that the analysis for finding $0^n$ holds similarly.

Since the objective vector in $F_I^*$ will always be maintained by Lemma~\ref{lemma:nsga2}, we could use the conclusion of Lemma 4 in~\cite{doerr2023nsgaojzj}, which shows that if an objective vector in $F_I^*$ has been found, the expected number of generations of the first phase is $O(n\log n)$. Since we have proven in the first phase of Theorem~\ref{thm:sms1} that there is at least one initial solution $\bm{x}$ such that $\bm{f}(\bm{x}) \in F_I^*$ with a probability of $1 - \exp(-\Omega(n))^\mu$, we can then derive that the expected number of generations of phase~1 is $O(n \log n)$.

Next, we consider the second phase. We also use the analysis method of $M+1$ jumps in Theorem~\ref{thm:sms1}. Specifically, by employing SPU, any solution (including dominated ones) can survive into the next generation with a probability of at least $p_s$. We also assume a ``jump" to be an event where a solution $\bm{x}$ with $|\bm{x}|_1 \in [n-k .. n-1]$ is selected, and a new dominated solution closer to $1^n$ is generated and preserved. We refer to the dominated solutions along the multiple jumps as ``stepping stones" and assume that $1^n$ can be reached through $M$ stepping stones. We consider $M+1$ consecutive jumps from $\bm{x}$ with $n-k$ 1-bits to the Pareto optimal solution $1^n$. Any failure during the intermediate jumps will result in restarting the process from the beginning. Note that in each generation, the probability of selecting a specific parent solution is changed from $1/\mu$ to $1$. 
Let $p_i$ ($1\le i\le M-1$) denote the probability that, after selecting the $i$-th stepping-stone solution as a parent solution, the next stepping-stone solution is successfully generated. Let $p_0$ denotes the probability of selecting the solution $\bm{x}$ with $n-k$ 1-bits, and generating the first stepping-stone solution. For the final jump, the Pareto optimal solution $1^n$ found must be preserved, and the jump occurs with a probability of $p_M$. Then, the probability of a successful trail of $M+1$ consecutive jumps is $p_s^M\prod_{i=0}^M p_i$. Similar to the analysis in Eq.~\eqref{eq:product_p_i}, we have
\begin{equation}
\begin{aligned}
  p_s^M\prod_{i=0}^M p_i \ge& \frac{p_s^M(1-p_c)^{M+1}}{e^{M+1}n^k} \cdot \frac{k!}{k_0! \cdot k_1!\cdot k_2! \cdots k_M!}.
\end{aligned}
\end{equation}
By considering all possible jump distances, the probability of a successful trail of $M+1$ consecutive jumps is at least
\begin{equation}
\begin{aligned}\label{eq:prob_attempts_nsga_ii}
  &\frac{p_s^M(1-p_c)^{M+1}}{e^{M+1}n^k} \cdot \sum_{k_0+\cdots + k_M = k} \frac{k!}{k_0! \cdot k_1!\cdot k_2! \cdots k_M!}\\
  &= \frac{p_s^M(1-p_c)^{M+1}(M+1)^k}{e^{M+1}n^k},
\end{aligned}
\end{equation}
where the last equality holds by the \textit{Multinomial Theorem}. Therefore, $1^n$ can be found in at most $(e n^k/(1-p_c)) \cdot (e/(p_s(1-p_c)))^M/ (M+1)^k$ trails of $M+1$ consecutive jumps in expectation. Because each trail needs at most $M+1$ generations, the expected number of generations for finding $1^n$ is at most
\begin{equation}
\begin{aligned}\label{eq:runtime_with_m_nsga_ii}
  &(M+1)\cdot \frac{e n^k}{(1-p_c)(M+1)^k} \big(\frac{e}{p_s(1-p_c)}\big)^M \\
  &= O\bigg(\frac{ n^k}{(M+1)^{k-1}}\big(\frac{e}{p_s(1-p_c)}\big)^M\bigg),
\end{aligned}
\end{equation}
Now, we need to determine an appropriate $M$ to achieve effective acceleration. Let $C \coloneqq e/(p_s(1-p_c))$. The following analysis is similar to that of $f(M)$ in Eq.~\eqref{f_M}. When $k\le e\ln C$, we take $M=0$, and then Eq.~\eqref{eq:runtime_with_m_nsga_ii} becomes $O(n^k)$. When $k> e\ln C$, we take $M = \lceil ((k-1)/\ln C)-1 \rceil $, and then Eq.~\eqref{eq:runtime_with_m_nsga_ii} becomes
\begin{equation}
\begin{aligned}
O\bigg(\frac{n^k}{(\frac{k-1}{\ln C})^{k-1}} \cdot C^{\frac{k-1}{\ln C}}\bigg) =O\bigg(n^k \cdot \big( \frac{e\ln C}{k} \big)^{k-1} \bigg),
\end{aligned}
\end{equation}
where the equality holds by $(k-1)/\ln C = \log_{C}^{e^{k-1}}$ and thus $C^{\frac{k-1}{\ln C}} = e^{k-1}$. Therefore, the expected number of generations of the second phase, i.e., finding $1^n$, is $O( n^k\cdot \min\{1,(e\ln C/k)^{k-1}\})$, where $C = e/(p_s(1-p_c))$. 

Combining the two phases, the total expected number of generations is $O(n\log n+ n^k\cdot \min\{1,(e\ln C/k)^{k-1}\}) = O(n^k\cdot \min\{1,(e\ln C/k)^{k-1}\})$, where $C = e/(p_s(1-p_c))$. The equality holds because $k\ge 2$ and $(en\ln C/k)^{k-1} \ge \log n$ as shown in the last paragraph in the proof of Theorem~\ref{thm:sms1}. For NSGA-II, each generation requires $\mu$ fitness evaluations. Therefore, the expected number of fitness evaluations is $O(\mu n^k\cdot \min\{1,(e\ln C/k)^{k-1}\})$. Thus, the theorem holds.	
\end{proof}

Bian \textit{et al.}~\shortcite{bian23stochastic} analyzed NSGA-II solving OneJumpZeroJump with $k > 8e^2$, $p_s = 1/4$, and $p_c = 0$, deriving an expected running time of $O(\mu \sqrt{k}(n/2)^k)$. In contrast, Theorem~\ref{thm:NSGA2-1} gives a tighter bound $O(\mu n^k\cdot (e\ln (4e)/ k)^{k-1})$, with an improvement ratio of $\Theta(\sqrt{k}(k/(2e\ln (4e)))^{k-1})$, which can be exponential when $k$ is large, e.g., $k=n/8$. Our analysis leads to improvement because, in the second phase of finding the extreme Pareto optimal solution $1^n$, we consider an optimal number $M+1$ of jumps across the gap of the dominated solution set (i.e., $\{\bm{x} \mid |\bm{x}|_1 \in [n-k+1.. n-1]\}$), 
rather than a continuous sequence of $k$ jumps from $|\bm{x}|_1=n-k$ to $n$ in~\cite{bian23stochastic}.

\subsection{An Archive is Provably Helpful}

In the last section, we showed that for NSGA-II solving \ojzj, the acceleration brought by SPU is $O(\mu n^k\cdot \min\{1,(e\ln C/k)^{k-1}\})$, where $C = e/(p_s(1-p_c))$, significantly improving the running time $O(\mu n^k)$ without SPU~\cite{doerr2023nsgaojzj}. Now, we show that using an archive, the result can be further improved.

Once a new solution is generated, the solution will be tested if it can enter the archive. If there is no solution in the archive that dominates the new solution, then the solution will be placed in the archive. 
Algorithmic steps incurred by adding an archive in NSGA-II are given as follows.
For NSGA-II in Algorithm~\ref{alg:nsgaii}, an empty set $A$ is initialized in line~1, the following lines
\begin{framed}\vspace{-0.8em}
  \begin{algorithmic}
  \FOR{$\bmx''\in P'$}
  \IF{$\not \exists \bmz \in A$ such that $\bmz \succ \bmx''$}
  \STATE $A \leftarrow (A \setminus\{\bmz \in A \mid \bmx'' \succeq \bmz\}) \cup \{\bmx''\}$
  \ENDIF
  \ENDFOR
\end{algorithmic}\vspace{-0.8em}
\end{framed}\noindent
are added after line~13, and the set $A$ instead of $P$ is returned in the last line. 

We prove in Theorem~\ref{thm:NSGA2-2} that if using an archive, a population size $\mu \ge 5$ is sufficient to guarantee the same running time bound of NSGA-II with SPU as Theorem~\ref{thm:NSGA2-1}. Note that Theorem~\ref{thm:NSGA2-1} requires $\mu \ge 4(\mu-2k+3)/(1-2p_s)$, which implies that using an archive can allow a small population size and thus bring speedup.

\begin{theorem}\label{thm:NSGA2-2}
  For NSGA-II solving OneJumpZeroJump with $n-2k=\Omega(n)$, if using SPU with survival probability $p_s\in [1/(2\mu),(\mu-4)/(2\mu)]$, the crossover probability $p_c = \Theta(1)$, the population size $\mu \ge 5$, and using an archive, then the expected running time for finding the whole Pareto front is $O(\mu n^k\cdot \min\{1,(e\ln C/k)^{k-1}\})$, where $C = e/(p_s(1-p_c))$. 
\end{theorem} 

\begin{proof}
We divide the running process of NSGA-II into three phases. The first phase starts after initialization and finishes until the two boundary solutions of $F_I^*$ (i.e., a solution with $k$ 1-bits and a solution with $n-k$ 1-bits) are found; the second phase starts after the first phase and finishes when $1^n$ and $0^n$ are both found; the third phase starts after the second phase and finishes until the whole Pareto front is found.

For the first phase, we prove that the expected number of generations for finding solution $\bm{x}$ such that $|\bm{x}|_1 = n-k$ is $O(n\log n)$, and the same bound holds for finding a solution with $k$ 1-bits analogously. First, we show that the maximum $f_1$ value among the Pareto optimal solutions in $P\cup P'$ will not decrease, where $P'$ is the set of offspring solutions. Let $J$ denote the set of Pareto optimal solutions in $P \cup P'$ with the maximum $f_1$ value. By definition, all solutions in $J$ are Pareto optimal, which means they all belong to $R_1$ in the non-dominated sorting procedure. Furthermore, since all solutions in $J$ have the maximum $f_1$ value, there is at least one solution $\bm{x}^* \in J$ that is a boundary solution in $R_1$, and has an infinite crowding distance. Note that only solutions in the first and the last positions can have infinite crowding distance. As \ojzj\ has two objectives, at most four solutions in $R_1$ can have infinite crowding distance. In Algorithm~\ref{alg:nsga-SPU}$, \lfloor 2\mu \cdot (1-p_s)\rfloor$ solutions are selected for competition, and $\lfloor \mu - 2\mu \cdot p_s \rfloor \ge 4$ of them will not be removed, because the survival probability $p_s \in [1/2\mu, (\mu-4)/2\mu]$. Thus, $\bm{x}^*\in J$ is among the best $4$ solutions in $R_1$, implying that $\bm{x}^*$ must be preserved in the next population. Thus, the maximum $f_1$ value among the Pareto optimal solutions in $P\cup P'$ will not decrease. Then the analysis of the first phase is similar to that of Theorem~\ref{thm:sms2}. The main difference is that the probability of selecting a specific parent solution is changed from $1/\mu$ to $1$ because any solution in the current population will generate an offspring solution by line~4 of Algorithm~\ref{alg:nsgaii}. Then, we can derive that the expected number of generations of phase~1 is $O(n\log n)$.

For the second phase, the analysis is the same as that of the second phase in Theorem~\ref{thm:NSGA2-1}. Before reaching the extreme Pareto optimal solution $1^n$, the boundary Pareto optimal solution with $n-k$ 1-bits is always the Pareto optimal solution with the maximum $f_1$ value and thus preserved in the population. Using SPU, the process of reaching $1^n$ can be constructed as jumping across the gap of the dominated solution set (i.e., $\{\bm{x} \mid |\bm{x}|_1 \in [n-k+1, n-1]\}$) through $M$ stepping stones, and any failure during the intermediate jumps will result in restarting the process from the solution with $n-k$ 1-bits. By setting proper values of $M$, this process leads to an expected number of generations, $O( n^k\cdot \min\{1,(e\ln C/k)^{k-1}\})$, where $C = e/(p_s(1-p_c))$.

Finally, we consider the third phase. The analysis of the third phase is just similar to that of Theorem~\ref{thm:sms2}. The only difference is that the probability of selecting $1^n$ or $0^n$ as a parent solution is changed from $1/\mu$ to $1$, because the fair selection selects each solution in the population once (the order of solution is random) to form the parent population. Then, we directly derive that the expected number of generations of phase~3 is $O(n\log n)$.

Combining the three phases, the total expected number of generations is $O(n\log n + n^k\cdot \min\{1,(e\ln C/k)^{k-1}\} + n\log n) = O(n^k\cdot \min\{1,(e\ln C/k)^{k-1}\})$, where $C = e/(p_s(1-p_c))$. The equality holds because $k\ge 2$ and $(en\ln C/k)^{k-1} \ge \log n$ as shown in the last paragraph in the proof of Theorem~\ref{thm:sms1}. For NSGA-II, each generation requires $\mu$ fitness evaluations. Therefore, the expected number of fitness evaluations is $O(\mu n^k\cdot \min\{1,(e\ln C/k)^{k-1}\})$. Thus, the theorem holds.
\end{proof}

 Comparing Theorems~\ref{thm:NSGA2-1} and~\ref{thm:NSGA2-2}, we find that if using SPU, the expected running time for NSGA-II solving \ojzj\ with or without an archive is both $O(\mu n^k\cdot (e\ln C/k)^{k-1})$, where $C = e/(p_s(1-p_c))$. The difference is that using an archive allows for a small constant population size $\mu$, which can bring an acceleration factor of $\Theta(n/(1-2p_s))$, as $\mu$ is required to be at least $4(n-2k+3)/ (1-2p_s)$ if not using an archive. Compared to SMS-EMOA, using an archive does not lead to an exponential reduction in the upper bound for NSGA-II. The reason is that NSGA-II employs the $(\mu + \mu)$ update mode and fair selection, allowing it to select and explore all solutions in each generation, thereby alleviating selection pressure. This results in a smaller constant $C = e/(p_s(1-p_c))$ independent of $\mu$ in Theorems~\ref{thm:NSGA2-1} and~\ref{thm:NSGA2-2}, compared to $C = e\mu /(p_s(1-p_c))$ in Theorems~\ref{thm:sms1} and~\ref{thm:sms2}. This also implies that NSGA-II achieves a smaller expected running time than SMS-EMOA, suggesting that the $(\mu+\mu)$ mode may be more suitable for SPU than the $(\mu+1)$ mode.

\section{Experiments}

In this section, we conduct experiments on OneJumpZeroJump and the well-known multi-objective travelling salesman problem (MOTSP) \cite{ribeiro2002study}. Table~\ref{tab:ojzj} presents the results of NSGA-II solving the OneJumpZeroJump problem with size $n \in \{10, 15, 20, 25, 30\}$ and $k=3$. The settings are: population size $\mu = 8$ with an archive, and $\mu = 8(n-2k+3)$ without; survival probability $p_s = 1/2$ when using SPU; and crossover probability $p_c = 1/2$. Table~\ref{tab:ojzj_sms} presents the results of SMS-EMOA solving the OneJumpZeroJump problem with size $n \in \{10, 15, 20, 25, 30\}$ and $k=3$. The settings are: population size $\mu = 5$ with an archive, and $\mu = 2(n-2k+4)$ without; survival probability $p_s = 1/2$ when using SPU; and crossover probability $p_c = 1/2$.
We present the average number of fitness evaluations over 200 runs for three configurations: SPU only, archive only, and SPU + archive. Results show that using SPU+archive significantly reduces running time compared to using only SPU, and using only archive.

\begin{table*}[th!]
  \centering
  \caption{The IGD~\protect\cite{coello2004study} results (mean and standard deviation) of the four variants of NSGA-II and SMS-EMOA on the 100-cities MOTSP instance clusAB from TSPLIB~\protect\cite{reinelt_tsplibtraveling_1991}. For each MOEA, the best mean is highlighted in bold.}
  \vspace{-0.3cm}
  \label{tab:igd_results_clusAB}
  \resizebox{\textwidth}{!}{
  \begin{tabular}{lcccc}
    \toprule
    \textbf{Algorithm} & \textbf{Original Algorithm} $\mu=2|\mathcal{PF}|$ & \textbf{Only SPU $\mu=2|\mathcal{PF}|$}   & \textbf{Only Archive $\mu=\frac{1}{4}|\mathcal{PF}|$} & \textbf{SPU+Archive $\mu=\frac{1}{4}|\mathcal{PF}|$} \\
    \midrule
    SMS-EMOA & 4.0923e+4 (5.69e+3) $\dagger$ & 3.8018e+4 (6.27e+3) $\dagger$ & 2.1939e+4 (2.03e+3) $\dagger$ & \textbf{1.2201e+4 (1.66e+3)} \\
    NSGA-II & 4.0017e+4 (2.55e+3) $\dagger$ & 3.2850e+4 (3.03e+3) $\dagger$ & 1.9363e+4 (2.21e+3) $\dagger$ & \textbf{9.8081e+3 (1.49e+3)} \\
    \bottomrule
  \end{tabular}}
  \small{``$\dagger$'' indicates that the result is significantly different from that of the SPU + archive algorithm (last column), at a 95\% confidence by the Wilcoxon rank-sum test.}
\end{table*}
\begin{table*}[th!]
  \centering
  \caption{The IGD~\protect\cite{coello2004study} results (mean and standard deviation) of the four variants of NSGA-II and SMS-EMOA on the 100-cities MOTSP instance kroAB from TSPLIB~\protect\cite{reinelt_tsplibtraveling_1991}. For each MOEA, the best mean is highlighted in bold.}
  \vspace{-0.3cm}
  \label{tab:igd_results_kroAB}
  \resizebox{\textwidth}{!}{
  \begin{tabular}{lcccc}
    \toprule
    \textbf{Algorithm} & \textbf{Original Algorithm} $\mu=2|\mathcal{PF}|$ & \textbf{Only SPU $\mu=2|\mathcal{PF}|$}   & \textbf{Only Archive $\mu=\frac{1}{4}|\mathcal{PF}|$} & \textbf{SPU+Archive $\mu=\frac{1}{4}|\mathcal{PF}|$} \\
    \midrule
    SMS-EMOA & 3.8339e+4 (2.18e+3) $\dagger$ & 3.7996e+4 (1.99e+3) $\dagger$ & 1.5185e+4 (1.63e+3)  & \textbf{1.4736e+4 (1.48e+3)} \\
    NSGA-II & 4.2965e+4 (1.85e+3) $\dagger$ & 3.5934e+4 (2.30e+3) $\dagger$ & 1.6889e+4 (1.82e+3) $\dagger$ & \textbf{1.2462e+4 (1.30e+3)} \\
    \bottomrule
  \end{tabular}}
  \small{``$\dagger$'' indicates that the result is significantly different from that of the SPU + archive algorithm (last column), at a 95\% confidence by the Wilcoxon rank-sum test.}
  \vspace{-0.4cm}
\end{table*}

\begin{table}[t!] 
  \centering
  \begin{tabular}{cccc}
    \toprule
    \textbf{size} $n$ & \textbf{SPU} & \textbf{Archive} & \textbf{SPU+Archive} \\
    \midrule
    10 & 5212.48 & 11733.52 & 5125.52 \\
    15 & 35653.92 & 57941.48 & 25808.24 \\
    20 & 106424.76 & 183201.24 & 65398.97 \\
    25 & 239294.88 & 415550.16 & 146384.08 \\
    30 & 479311.56 & 767354.84 & 284114.12 \\
    \bottomrule
  \end{tabular}
  \caption{Average number of fitness evaluations over 200 independent runs for NSGA-II solving OneJumpZeroJump with $k=3$.}\label{tab:ojzj}
  \vspace{-0.1cm}
\end{table}

\begin{table}[t!] 
  \centering
  \begin{tabular}{cccc}
    \toprule
    \textbf{size} $n$ & \textbf{SPU} & \textbf{Archive} & \textbf{SPU+Archive} \\
    \midrule
    10 & 8522.88 & 14213.79 & 5413.985 \\
    15 & 63693.26 & 66607.01 & 26275.59 \\
    20 & 214833.25 & 158766.44 & 72884.91 \\
    25 & 579966.42 & 323110.73 & 140995.59 \\
    30 & 1176418.18 & 663510.10 & 262400.4 \\
    \bottomrule
  \end{tabular}
  \caption{Average number of fitness evaluations over 200 independent runs for SMS-EMOA solving OneJumpZeroJump with $k=3$.}\label{tab:ojzj_sms}
  \vspace{-0.1cm}
\end{table}

MOTSP is a classic combinatorial optimization problem, which can be described as:
given a network $L=(V,C)$, where $V=\{v_{1},\dots,v_{D}\}$ is a set of $D$ nodes and $C=\{C_{j}: j\in \{1,\dots,m\}\}$ is a set of $m$ cost matrices
between nodes $(C_{j}: V \times V)$, the task is to find the Pareto optimal set of 
Hamiltonian cycles that minimize each of the $m$ cost objectives. For the MOTSP, the Pareto front size $|\mathcal{PF}|$ is known, 
allowing us to set commensurable population sizes $\mu$ for the compared MOEAs. We conduct experiments on two 100-cities instances (i.e., clusAB and kroAB) of MOTSP~\footnote{\url{https://webia.lip6.fr/~lustt/Research.html}} using NSGA-II and SMS-EMOA under four scenarios with varying population sizes: 1) the original algorithms, 2) using only SPU, 3) using only an archive, and 4) using SPU and an archive. For a fair comparison, each scenario used $1,000,000$ fitness evaluations over $30$ runs. Table~\ref{tab:igd_results_clusAB} and~\ref{tab:igd_results_kroAB} presents the results of a widely-used quality indicator, Inverted Generational Distance (IGD)~\cite{coello2004study} on the instance clusAB and kroAB, respectively. We used IGD since it can measure how well the obtained solution set represents the Pareto front~\cite{li2020evaluate}. 
Table~\ref{tab:igd_results_clusAB} shows that using SPU and an archive simultaneously can lead to the smallest value of IGD for both SMS-EMOA and NSGA-II, and statistically outperforms the other three scenarios. Moreover, with the SPU method, NSGA-II always achieves smaller IGD values than SMS-EMOA. 
Similar results can be observed with smaller population sizes. Specifically, Table~\ref{tab:igd_results_additional_1} and Table~\ref{tab:igd_results_additional_2}) reports the IGD values obtained under reduced population sizes—namely, $\mu = |\mathcal{PF}|/6$, $|\mathcal{PF}|/8$, $|\mathcal{PF}|/10$, and $|\mathcal{PF}|/12$—on the clusAB and kroAB instances, respectively. The results consistently show that the combination of SPU and an archive maintains superior performance.

\begin{table}[ht!]
  \centering
  \begin{tabular}{lll}
    \toprule
    \textbf{} & \textbf{Only Archive}  & \textbf{SPU+Archive } \\
    \midrule
    $\mu=|\mathcal{PF}|/6$ & 1.60e+4 (2.05e+3) & \textbf{8.29e+3 (4.32e+2)}  \\
    $\mu=|\mathcal{PF}|/8$ & 1.60e+4 (1.43e+3) & \textbf{8.92e+3 (1.23e+3)}  \\
    $\mu=|\mathcal{PF}|/10$ & 1.69e+4 (2.10e+3) & \textbf{1.26e+4 (1.08e+3)}  \\
    $\mu=|\mathcal{PF}|/12$ & 1.34e+4 (1.53e+3) & \textbf{1.14e+4 (1.65e+3)}  \\
    \bottomrule
  \end{tabular}
  \caption{The IGD~\protect\cite{coello2004study} results (mean and standard deviation) of the two variants of NSGA-II on the 100-cities MOTSP instance clusAB from TSPLIB~\protect\cite{reinelt_tsplibtraveling_1991}. For each MOEA, the best mean is highlighted in bold.}\label{tab:igd_results_additional_1}
\end{table}

\begin{table}[ht!]
  \centering
  \begin{tabular}{lll}
    \toprule
    \textbf{} & \textbf{Only Archive}  & \textbf{SPU+Archive } \\
    \midrule
    $\mu=|\mathcal{PF}|/6$ & 1.83e+4 (2.00e+3) & \textbf{1.22e+4 (1.77e+3)}  \\
    $\mu=|\mathcal{PF}|/8$ & 1.60e+4 (1.83e+3) & \textbf{1.16e+4 (1.41e+3)}  \\
    $\mu=|\mathcal{PF}|/10$ & 1.69e+4 (2.10e+3) & \textbf{1.26e+4 (1.08e+3)}  \\
    $\mu=|\mathcal{PF}|/12$ & 1.58e+4 (1.53e+3) & \textbf{1.37e+4 (1.83e+3)}  \\
    \bottomrule
  \end{tabular}
  \caption{The IGD~\protect\cite{coello2004study} results (mean and standard deviation) of the two variants of NSGA-II on the 100-cities MOTSP instance kroAB from TSPLIB~\protect\cite{reinelt_tsplibtraveling_1991}. For each MOEA, the best mean is highlighted in bold.}
  \label{tab:igd_results_additional_2}
\end{table}



\section{Conclusion}

This paper analytically shows that SPU in MOEAs needs an archive to better leverage its exploration ability. We prove that for NSGA-II and SMS-EMOA solving OneJumpZeroJump, introducing an archive for SPU can address the dilemma of large population size and may provide an extra exponential speedup. The reason is that SPU requires a large population to preserve the best solutions found, 
while an archive enables a small population size, increasing the chance of selecting inferior but undeveloped solutions.
We also find that the $(\mu+\mu)$ update mode may be more suitable for SPU than the $(\mu+1)$ update mode. 
Another contribution lies in improving the running time bounds for SMS-EMOA and NSGA-II solving OneJumpZeroJump using SPU.
These theoretical findings are empirically validated on OneJumpZeroJump and the MOTSP. We hope our work may provide some theoretical evidence for the attempts of designing new MOEAs that separate the exploration (via the evolutionary population) and elitist solution preservation (via an external archive), such as in non-elitist or less elitist MOEAs~\cite{tanabe2019non,liang2023non}.

\section*{Acknowledgements} 
This work was supported by the National Science and Technology Major Project (2022ZD0116600), the National Science Foundation of China (62276124), and the Fundamental Research Funds for the Central Universities (14380020). Chao Qian is the corresponding author.

\bibliographystyle{named}
\bibliography{SPU_archive}

\end{document}